\RequirePackage{fix-cm}
\documentclass[twocolumn]{svjour3}       
\smartqed  
\usepackage{graphicx}
\usepackage{multirow}
\usepackage{enumitem}

\usepackage{amsmath,amsthm,amssymb}
\usepackage[toc,page]{appendix}

\newtheorem{prop}{Proposition}
\begin{document}

\title{Versatile Auxiliary Classifier with Generative Adversarial Network (VAC+GAN), Multi Class Scenarios 
}
\subtitle{Training Conditional Generators}


\author{Shabab Bazrafkan        \and
        Peter Corcoran 
}


\institute{S. Bazrafkan \at
              Cognitive, Connected \& Computational Imaging Research\\ National University of Ireland Galway \\
              Tel.: +353-83-466-7835\\
              \email{s.bazrafkan1@nuigalway.ie}           
                        \and
           P. Corcoran \at
             Cognitive, Connected \& Computational Imaging Research\\ National University of Ireland Galway
}

\date{Received: date / Accepted: date}

\maketitle

\begin{abstract}
Conditional generators learn the data distribution for each class in a multi-class scenario and generate samples for a specific class given the right input from the latent space. In this work, a method known as \textquotedblleft Versatile Auxiliary Classifier with Generative Adversarial Network\textquotedblright for multi-class scenarios is presented. In this technique, the Generative Adversarial Networks (GAN)'s generator is turned into a conditional generator by placing a multi-class classifier in parallel with the discriminator network and backpropagate the classification error through the generator. This technique is versatile enough to be applied to any GAN implementation. The results on two databases and comparisons with other method are provided as well.  
\end{abstract}

\section{Introduction}
\label{intro}
With emerge of affordable parallel processing hardware, it became almost impossible to find any aspect of Artificial Intelligence (AI) that Deep Learning (DL) has not been applied to \cite{CEmag1}. DL provides superior outcomes on classification and regression problems compared to classical machine learning methods. The impact of DL is not limited to such problems, but also generative models are taking advantage of these techniques in learning data distribution for big data scenarios where classical methods fail to provide a solution. Generative Adversarial Networks (GAN) \cite{GAN} utilise Deep
Neural Network capabilities and are able to estimate
the data distribution for large size problems. These
models comprise two networks, a generator, and a discriminator. The generator makes random samples from
a latent space, and the discriminator determines whether
the sample is adversarial, made by the generator, or
is genuine image coming from the dataset. GANs are
successful implementations of deep generative models,
and there are multiple variations such as WGAN \cite{WGAN},
EBGAN \cite{EBGAN}, BEGAN \cite{BEGAN}, ACGAN \cite{ACGAN}, and DCGAN \cite{DCGAN},
which have evolved from the original GAN by altering
the loss function and/or the network architecture.
Variational Autoencoders (VAE) \cite{VAE} are the other successful implementation of deep generative models. In
these models the bottleneck of a conventional autoencoder is considered as the latent space of the generator, i.e., the samples are fed to an autoencoder,
and besides the conventional autoencoder’s loss function, the KullbackLeibler (KL) divergence between the
distribution of the data at the bottleneck is minimized
compared to a Gaussian distribution. In practice, this
is achieved by adding the KL divergence term to the
means square error of the autoencoder network. The
biggest downside to VAE models is their blurry outputs due to the mean square error loss \cite{VAE2}.
PixelRNN and PixelCNN \cite{RNN} are other famous implementations of the deep neural generative models. PixelRNN is made of 2-dimensional LSTM units, and in PixelCNN, a Deep Convolutional Neural Network is utilized to estimate the distribution of the data.\\
Training conditional generators are one of the most appealing applications of GAN.  Conditional GAN (CGAN) \cite{CGAN} and Auxiliary Classifier GAN (ACGAN) \cite{ACGAN} are among the most utilized schemes for this purpose.  Wherein the CGAN approach uses the auxiliary class information alongside with partitioning the latent space and ACGAN improves the CGAN idea by introducing a classification loss which back-propagates through the discriminator and generator network. The CGAN method is versatile enough to apply to every variation of GAN.  But ACGAN is restricted to a specific loss function which decreases its adaptivity to other GAN varieties.\\
In \cite{VACGAN}, the ACGAN technique is extended to be applicable to any GAN implementation for binary problems (2 class scenarios). The technique is known as Versatile Auxiliary Classifier with Generative Adversarial Network (VAC+GAN) and is implemented by placing a classifier in parallel with the discriminator and back-propagate the classification error through the generator alongside the GAN's loss.\\
This work expand the original VAC+GAN \cite{VACGAN}  idea to multi-class scenarios. In this approach, the classifier is trained independently from the discriminator which gives the opportunity of applying it to any variation of GAN. The main contribution of VAC+GAN is its versatility, and proofs are provided to show the applicability of the method regardless of the GAN structure or loss functions.\\
In the next section the VAC+GAN for multi-class scenarios is explained. And in the third section the implementations of the ACGAN and VAC+GAN is presented alongside with the comparisons with other methods. The discussions and future works are given in the last section.
\section{Versatile Auxiliary Classifier + Generative Adversarial Network (VAC+GAN)}
\label{sec:2}
The concept proposed in this research is to place a classifier network in parallel with the Discriminator. The classifier accepts the samples from the generator, and the classification error is back-propagated through the classifier and the generator. The model structure is shown in figure \ref{fig:1}.
\begin{figure}
  \includegraphics[width=\columnwidth]{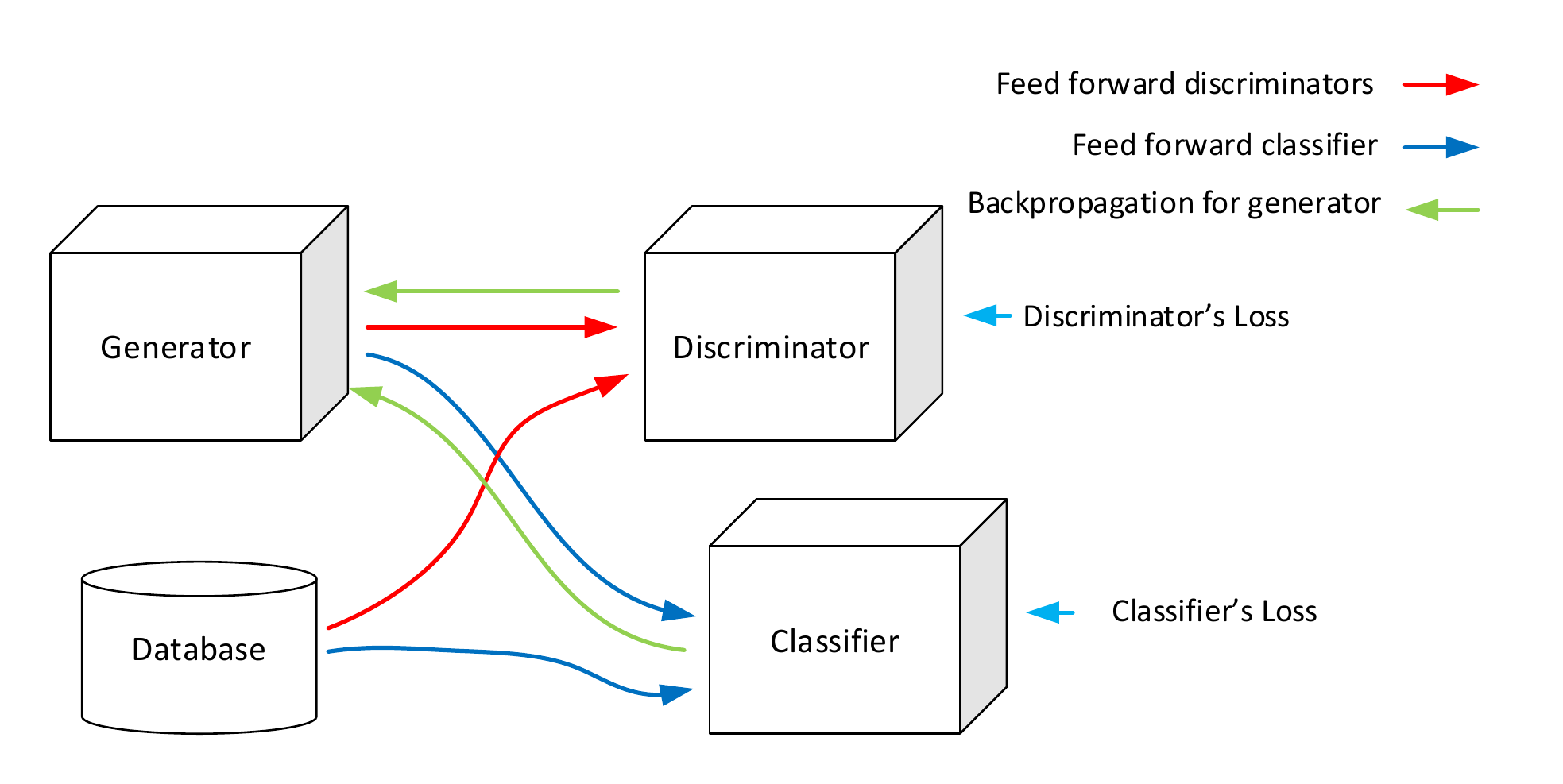}
\caption{The presented model for training conditional deep generators.}
\label{fig:1}       
\end{figure}

In this section it is shown that by placing a classifier at the output of the generator and minimizing the categorical cross-entropy as the classifiers loss, the Jensen-Shannon Divergence between all the classes is increased. The terms used in the mathematical proofs are as follows:
\begin{enumerate}
\item $N$ is the number of the classes.
\item The latent space $Z$ is partitioned in to $\{{Z_1},{Z_2},\ldots,{Z_N}\}$ subsets. This means that $\{{Z_1},{Z_2},\ldots,{Z_N}\}$ are disjoint and their union is equal to the $Z$-space.
\item $C$ is the classifier function.
\item $\mathcal{L}_{ce}$ is the binary cross-entropy loss function.
\item $\mathcal{L}_{cce}$ is the categorical cross-entropy loss function.
\end{enumerate}
\begin{prop}
In the multiple classes case, the classifier $C$ has $N$ outputs, where $N$ is the number of the classes. In this approach, each output of the classifier corresponds to one class. For a fixed Generator and Discriminator, the optimal output for class $c$ ($c$'th output) is:
\begin{equation}
C_{G,D}^{*}(c)=\frac{p_{X_c}({\bf x})}{\sum_{i=1}^{N}p_{X_i}(\bf x)}
\label{eq:prop2}
\end{equation}
\end{prop}
\begin{proof}
Considering just one of the outputs of the classifier, the categorical cross-entropy can be reduced to binary cross-entropy given by
\begin{equation}
\begin{split}
\mathcal{L}_{ce}(C(c)) = &-\mathbb{E}_{{\bf z}\sim p_{Z_c}({\bf z})}\big[\log\big(C(G({\bf z}))\big)\big]\\
&-\mathbb{E}_{{\bf z}\sim\sum_{i\neq c}p_{Z_i}(\bf z)}\big[1-\log\big(C(G({\bf z}))\big)\big]\\
\end{split}
\label{eq:p4e1}
\end{equation}
which is equal to
\begin{equation}
\begin{split}
\mathcal{L}_{ce}(C(c)) &=\int\Big(p_{Z_c}({\bf z})\log\big(C(G({\bf z}))\big)\\
&+ \big(\sum_{i\neq c}p_{Z_i}({\bf z})\big)\log\big(1-C(G({\bf z}))\big)d{\bf z}\Big)
\end{split}
\label{eq:p4e2}
\end{equation}
By considering $G({\bf z}_i)={\bf x}_i$ we have
\begin{equation}
\begin{split}
\mathcal{L}_{ce}(C(c)) &=\int \Big(p_{X_c}({\bf x})\log(C({\bf x}))\\&+\big(\sum_{i\neq c}p_{X_i}({\bf x})\big)\log(1-C({\bf x})) d{\bf x}\Big)
\end{split}
\label{eq:p4e3}
\end{equation}
The function $f\rightarrow m\log(f)+n\log(1-f)$ gets its maximum at $\frac{m}{m+n}$ for any $(m,n)\in \mathbb{R}^2 \setminus \{0,0\}$, concluding the proof.
\end{proof}
\begin{theorem}
The maximum value for $\mathcal{L}_{cce}(C)$ is $N\log(N)$ and is achieved if and only if $p_{X_1}=p_{X_2}=\ldots=p_{X_N}$.
\end{theorem}
\begin{proof}
The categorical cross-entropy is given by
\begin{equation}
\begin{split}
\mathcal{L}_{cce} &= -\sum_{i=1}^{N}\mathbb{E}_{{\bf z}\sim p_{Z_i}({\bf z})}\big[\log\big(C(G({\bf z}))\big)\big]\\
&=-\sum_{i=1}^{N}\int p_{X_i}({\bf x})\log(C({\bf x}))d{\bf x}
\end{split}
\label{eq:p5e1}
\end{equation}
From equation \ref{eq:prop2} we have
\begin{equation}
\begin{split}
&\mathcal{L}_{cce} = -\sum_{i=1}^{N}\Bigg(\int p_{X_i}({\bf x})\log\bigg(\frac{p_{X_i}({\bf x})}{\sum_{j=1}^{N}p_{j}({\bf x})}\bigg)d{\bf x}\Bigg)\\
&=-\sum_{i=1}^{N}\Bigg(\int p_{X_i}({\bf x})\log\bigg(\frac{p_{X_i}({\bf x})}{\sum_{j=1}^{N}\frac{p_{j}({\bf x})}{N}}\bigg)d{\bf x}\Bigg)+N\log(N)\\
&=Nlog(N)-\sum_{i=1}^{N}KL\bigg(p_{X_i}({\bf x})\bigg|\bigg|\sum_{j=1}^{N}\frac{p_{X_j}({\bf x})}{N}\bigg)
\end{split}
\label{eq:p5e2}
\end{equation}
Where $KL$ is the Kullback-Leibler divergence, which is always positive or equal to zero.\\
Now consider $p_{X_1}=p_{X_2}=\dots=p_{X_N}$. From \ref{eq:p5e2} we have
\begin{equation}
\begin{split}
\mathcal{L}_{cce} &= N\log(N)-\sum_{i=1}^{N}KL\Big(p_{X_i}({\bf x})\Big|\Big|p_{X_i}({\bf x})\Big)\\&=N\log(N)
\end{split}
\end{equation}
concluding the proof.
\end{proof}
\begin{theorem}
Minimizing $\mathcal{L}_{cce}$ increases the \\Jensen-Shannon Divergence between $p_{X_1},p_{X_2},\ldots,p_{X_N}$
\end{theorem}
\begin{proof}
From equation \ref{eq:p5e2} we have
\begin{equation}
\begin{split}
&\mathcal{L}_{cce}=N\log(N)\\
&\begin{split}-\int\sum_{i=1}^{N}\Bigg(p_{X_i}({\bf x})&\bigg[\log(p_{X_i}({\bf x}))\\&-\log\Big(\sum_{j=1}^{N}\frac{p_{X_j}(\bf x)}{N}\Big)\bigg]\Bigg)d{\bf x}\end{split}
\end{split}
\label{eq:p6e1}
\end{equation}
Which can be rewritten as
\begin{equation}
\begin{split}
\mathcal{L}_{cce}&=N\log(N)-\sum_{i=1}^{N}\bigg(\int p_{X_i}({\bf x})\log(p_{X_i}({\bf x}))d{\bf x}\bigg)\\
&+\int\underbrace{\sum_{i=1}^{N}\Bigg(p_{X_i}({\bf x})\log\bigg(\sum_{j=1}^{N}\frac{p_{X_j}({\bf x})}{N}\bigg)\Bigg)}_{\big(\sum_{i=1}^{N}p_{X_i}({\bf x})\big)\bigg(\log\Big(\sum_{j=1}^{N}\frac{p_{X_{j}({\bf x})}}{N}\Big)\bigg)}d{\bf x}
\end{split}
\label{eq:p6e2}
\end{equation}
Which is equal to
\begin{equation}
\begin{split}
&\mathcal{L}_{cce}=N\log(N)\\&-N\sum_{i=1}^{N}\bigg(\frac{1}{N}\int p_{X_i}({\bf x})\log(p_{X_i}({\bf x}))d{\bf x}\bigg)\\
&+N\int\bigg(\sum_{i=1}^{N}\frac{p_{X_i}({\bf x})}{N}\bigg)\bigg(\log\Big(\sum_{j=1}^{N}\frac{p_{X_{j}({\bf x})}}{N}\Big)\bigg)d{\bf x}
\end{split}
\label{eq:p6e3}
\end{equation}
This equation can be rewritten as
\begin{equation}
\begin{split}
\mathcal{L}_{cce} &= N\log(N)\\&-\Bigg[H\bigg(\sum_{i=1}^{N}\frac{1}{N}p_{X_i}({\bf x})\bigg)-\sum_{i=1}^{N}\frac{1}{N}H\big(p_{X_i}({\bf x})\big)\Bigg]
\end{split}
\label{eq:p6e4}
\end{equation}
wherein the $H(p)$ is the Shannon entropy of the distribution $p$.\\
The Jensen Shannon divergence between $N$ distributions $p_1,p_2,\ldots,p_N$, is defined as
\begin{equation}
\begin{split}
JSD_{\pi_1,\pi_2,\ldots,\pi_N}\Big(p_1,p_2,\ldots,p_N\Big)&=H\bigg(\sum_{i=1}^{N}\pi_i p_i\bigg)\\&-\sum_{i=1}^{N}\pi_iH(p_i)
\end{split}
\label{eq:p6e5}
\end{equation}
From equations \ref{eq:p6e4} and \ref{eq:p6e5} we have
\begin{equation}
\begin{split}
\mathcal{L}_{cce} &= N\log(N)\\&-N~JSD_{\frac{1}{N},\frac{1}{N},\ldots,\frac{1}{N}}\Big(p_{X_1}({\bf x}),p_{X_2}({\bf x}),\ldots,p_{X_N}({\bf x})\Big)
\end{split}
\end{equation}
Minimizing $\mathcal{L}_{cce}$ is increasing the JSD term, concluding the proof.
\end{proof}
In this section it has been shown that by placing a classifier at the output of the generator and back-propagate the classification error throughout the generator one can increase the dis-similarity between the classes for generator and therefore train a deep generator that can produce class specified samples. In the next section the proposed idea is implemented for multi-class cases and also compared with state of the art methods.
\section{Experimental Results}
\label{sec:3}
In this section, two main experiments are explained to show the effectiveness of VAC+GAN. The first one is on MNIST database and visual comparisons with CGAN, CDCGAN and ACGAN is presented. The second experiment is on CFAR10 dataset and the classification error is compared against ACGAN method.
All the networks are trained in Lasagne \cite{LASAGNE} on top of Theano \cite{THEANO} library in Python, unless stated otherwise.

\subsection{MNIST}
\label{sec3subsec:1}
In this experiment, the performance of the proposed method is investigated on MNIST database. MNIST ("Modified National Institute of Standards and Technology") is known as the "hello world" dataset of computer vision. It is a historically significant image classification benchmark introduced in 1999, and there has been a considerable amount of research published on MNIST image classification. MNIST contains 60,000 training images and 10,000 test images, both drawn from the same distribution. It consists of $28\times 28$ pixel images of handwritten digits. Each image is assigned a single truth label digit from $[0,9]$. \\
The proposed method has been applied to the DCGAN scheme. The Generator, Discriminator and the Classifier used in this experiment are given in tables \ref{tab:2}, \ref{tab:3} and \ref{tab:4} respectively.
\begin{table}
\caption{the generator structure for the MNIST+DCGAN experiment. All deconvolution layers are using (2,2) padding with stride (2,2).}
\label{tab:2}       
\begin{tabular}{|l|l|l|l|}
\hline
Layer & Type & kernel & Activation  \\
\hline
Input & Input$(10)$ & -- & -- \\
\hline
Hidden 1 & Dense & $1024$& ReLU \\
\hline
BatchNorm 1 &-- &-- &--\\
\hline
Hidden 2 & Dense & $128\times 7 \times 7$& ReLU \\
\hline
BathNorm 2 &-- &-- &--\\
\hline
Hidden 3 & Deconv & $5\times 5$ (64ch)& ReLU \\
\hline
BathNorm 3 &-- &-- &--\\
\hline
Output & Deconv & $5\times 5$ (1ch)& Sigmoid \\
\hline
\end{tabular}
\end{table}
\begin{table}
\caption{the discriminator structure for the MNIST+DCGAN experiment. All convolution layers are using (2,2) padding with stride (2,2).}
\label{tab:3}       
\begin{tabular}{|l|l|l|l|}
\hline
Layer & Type & kernel & Activation  \\
\hline
Input & Input & -- & -- \\
\hline
Hidden 1 & Conv & $5 \times 5$ (64 ch)& LeakyR(0.2) \\
\hline
BatchNorm 1 &-- &-- &--\\
\hline
Hidden 2 & Conv & $5 \times 5$ (128 ch)& LeakyR(0.2) \\
\hline
BathNorm 2 &-- &-- &--\\
\hline
Hidden 3 & Dense & 1024 & LeakyR(0.2) \\
\hline
Output & Dense & 1 & Sigmoid \\
\hline
\end{tabular}
\end{table}
\begin{table}
\caption{the classifier structure for the MNIST+DCGAN experiment.}
\label{tab:4}       
\begin{tabular}{|l|l|l|l|}
\hline
Layer & Type & Kernel & Activation  \\
\hline
Input & Input$(28\times 28)$ & -- & -- \\
\hline
Hidden 1 & Conv & $3\times3$(16 ch)& ReLU \\
\hline
Pool 1 &Max pooling &$2\times2$ &--\\
\hline
Hidden 2 & Conv & $3\times3$(8 ch)& ReLU \\
\hline
Pool 2 &Max pooling &$2\times2$ &--\\
\hline
Hidden 3 & Dense & 1024& ReLU \\
\hline
Output & Dense & 10 & Softmax \\
\hline
\end{tabular}
\end{table}
And the loss function for the proposed method (VAC+GAN) is given by:
\begin{equation}
\begin{split}
&L_g = \vartheta \cdot BCE(G(z|c),1)+ \zeta \cdot CCE\\
&L_d = BCE(x,1)+BCE(G(z|c),0)
\end{split}
\end{equation}
where, $L_g$, and $L_d$ are the generator and discriminator losses respectively, $G$ is the generator function, $BCE$ is the binary cross-entropy loss for discriminator and $CCE$ is the categorical cross-entropy loss for the classifier. In this experiment, $\vartheta$ and $\zeta$ are equal to 0.2 and 0.8 respectively.\\
The optimizer used for training the generator and discriminator is ADAM with learning rate, $\beta_1$ and $\beta_2$ equal to 0.0002, 0.5 and 0.999 respectively. And the classifier is optimized using nestrov momentum gradient descent with learning rate and momentum equal to 0.01 and 0.9 respectively.
The results of the conditional generators trained using Conditional GAN (CGAN)\footnote{https://github.com/znxlwm/tensorflow-MNIST-cGAN-cDCGAN}, Conditional DCGAN (CDCGAN)\footnote{https://github.com/znxlwm/tensorflow-MNIST-cGAN-cDCGAN}, ACGAN\footnote{https://github.com/buriburisuri/ac-gan}, and proposed method (VAC+GAN) on MNIST dataset are shown in figures \ref{fig:10}, \ref{fig:11}, \ref{fig:12}\footnote{https://github.com/buriburisuri/ac-gan/blob/master/png/sample.png},and \ref{fig:13} respectively.\\
\begin{figure}
  \includegraphics[width=\columnwidth]{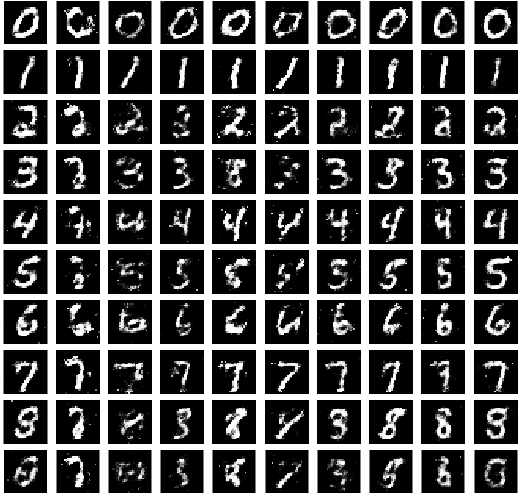}
\caption{Samples drawn from conditional generator trained using CGAN scheme on MNIST dataset. each row corresponds to one class.}
\label{fig:10}       
\end{figure}
\begin{figure}
  \includegraphics[width=\columnwidth]{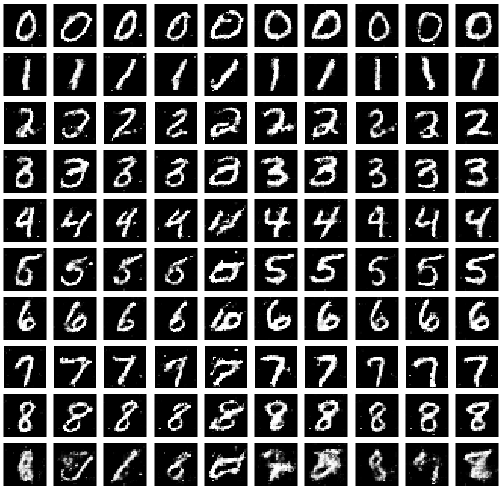}
\caption{Samples drawn from conditional generator trained using CDCGAN scheme on MNIST dataset. each row corresponds to one class.}
\label{fig:11}       
\end{figure}
\begin{figure}
  \includegraphics[width=\columnwidth]{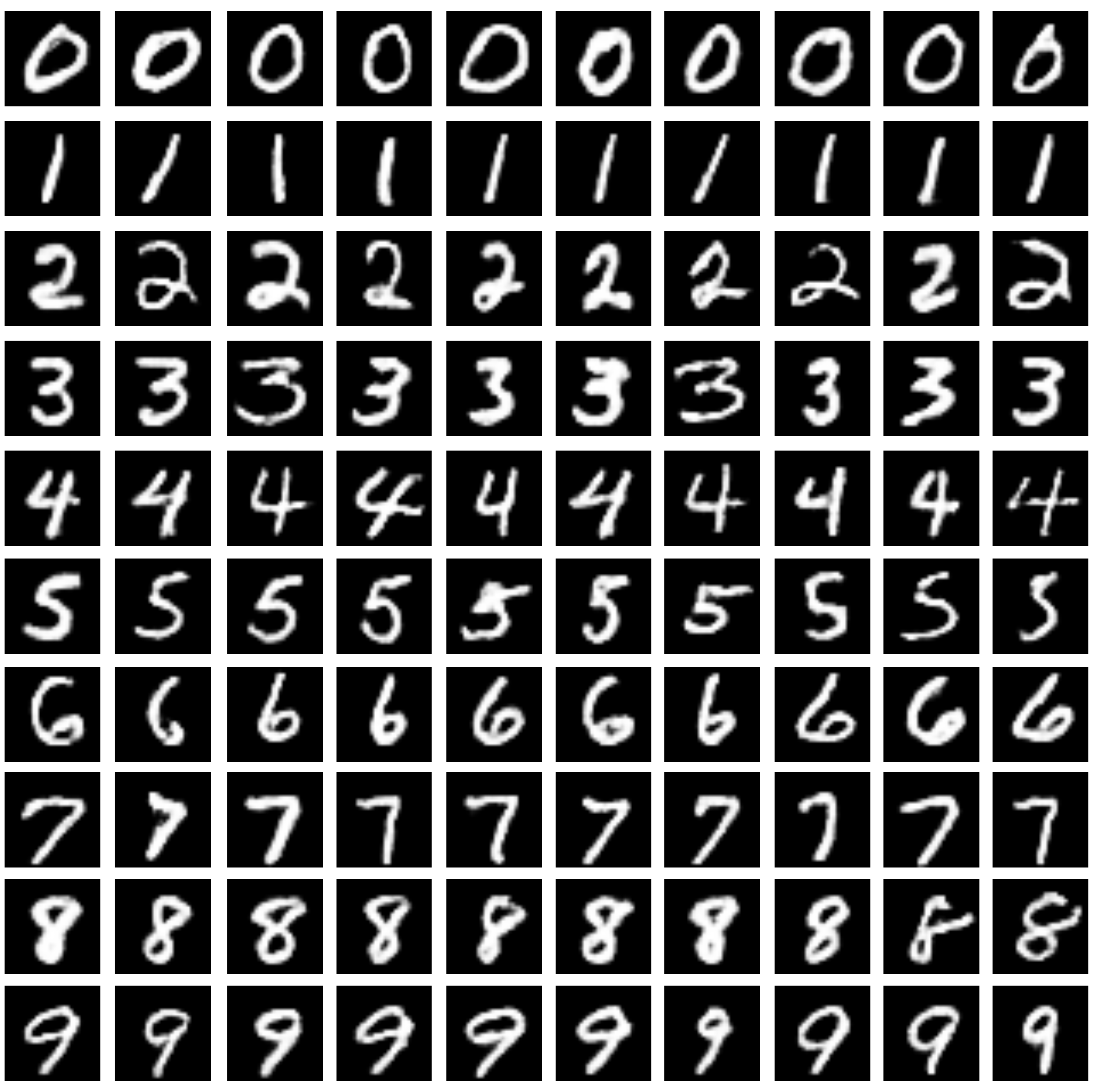}
\caption{Samples drawn from conditional generator trained using ACGAN scheme on MNIST dataset. each row corresponds to one class.}
\label{fig:12}       
\end{figure}
\begin{figure}
  \includegraphics[width=\columnwidth]{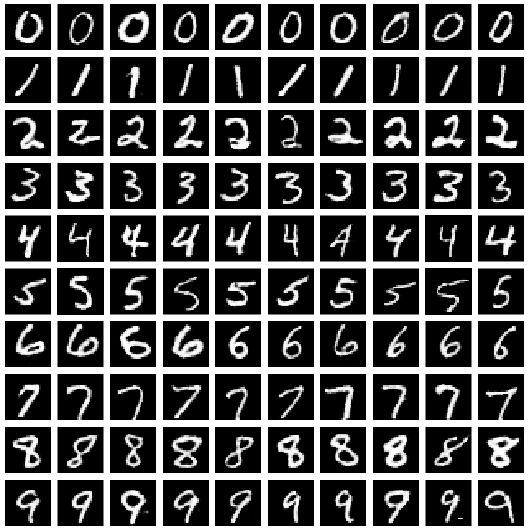}
\caption{Samples drawn from conditional generator trained using proposed scheme (VAC+GAN) on MNIST dataset. each row corresponds to one class.}
\label{fig:13}       
\end{figure}
As it is shown in these figures the presented method gives superior results compare to CGAN and CDCGAN while using the exact same structure of generator as in CDCGAN. The results are comparable with ACGAN and the difference here is that this method is more versatile and can be applied to any GAN model regardless of model architecture and loss function.
\subsection{CFAR10}
The CFAR10 database \cite{CFAR10} consists of 60000 images in 10 classes wherein 50000 of these images are for training and 10000 for testing purposes. The next experiment is comparing ACGAN\footnote{https://github.com/King-Of-Knights/Keras-ACGAN-CIFAR10} to VAC+GAN method on generating images and also the classification accuracy of these methods are compared. Networks utilized in this experiment are shown in tables \ref{tab:cfar1},\ref{tab:cfar2} and \ref{tab:cfar3} correspond to generator, discriminator\footnote{https://github.com/King-Of-Knights/Keras-ACGAN-CIFAR10/blob/master/cifar10.py} and classifier respectively. The same generator and discriminator architectures have been used in both implementations to obtain fair comparisons.
\begin{table}
\caption{the generator structure for the CFAR10 experiment. All deconvolution layers are using 'SAME' padding with stride (2,2).}
\label{tab:cfar1}       
\begin{tabular}{|l|l|l|l|}
\hline
Layer & Type & kernel & Activation  \\
\hline
Input & Input & -- & -- \\
\hline
Hidden 1 & Dense & $384\time4\times4$& ReLU \\
\hline
Reshape & Reshape & 384ch $4\times4$ &--\\
\hline
Hidden 2 & DeConv & $5 \times 5$ (192 ch)& ReLU \\
\hline
BathNorm 2 &-- &-- &--\\
\hline
Hidden 3 & DeConv & $5 \times 5$ (96 ch) & ReLU \\
\hline
BathNorm 3 &-- &-- &--\\
\hline
Output & DeConv & $5 \times 5$ (3 ch)& tanh \\
\hline
\end{tabular}
\end{table}

\begin{table}
\caption{the discriminator structure for the CFAR10 experiment. All deconvolution layers are using 'SAME' padding with kernel size $3\times3$, $st$ stands for stride size and MBDisc is Mini Batch Discrimination layer explained in \cite{MBDISC}.}
\label{tab:cfar2}       
\begin{tabular}{|l|l|l|l|}
\hline
Layer & Type & kernel & Activation  \\
\hline
Input & Input & $32\times32\times3$ & -- \\
\hline
Gaussian & Noise & $\sigma = 0.05$& -- \\
\hline
Hidden 1 & Conv & 16ch $st(2,2)$ & LeakyR(0.2)\\
\hline
DropOut 1 & DropOut & $p=0.5$& -- \\
\hline
Hidden 2 & Conv & 32ch $st(1,1)$ & LeakyR(0.2)\\
\hline
BathNorm 1 &-- &-- &--\\
\hline
DropOut 2 & DropOut & $p=0.5$& -- \\
\hline
Hidden 3 & Conv & 64ch $st(2,2)$ & LeakyR(0.2)\\
\hline
BathNorm 2 &-- &-- &--\\
\hline
DropOut 3 & DropOut & $p=0.5$& -- \\
\hline
Hidden 4 & Conv & 128ch $st(1,1)$ & LeakyR(0.2)\\
\hline
BathNorm 3 &-- &-- &--\\
\hline
DropOut 4 & DropOut & $p=0.5$& -- \\
\hline
Hidden 5 & Conv & 256ch $st(2,2)$ & LeakyR(0.2)\\
\hline
BathNorm 4 &-- &-- &--\\
\hline
DropOut 5 & DropOut & $p=0.5$& -- \\
\hline
Hidden 6 & Conv & 512ch $st(1,1)$ & LeakyR(0.2)\\
\hline
BathNorm 5 &-- &-- &--\\
\hline
DropOut 6 & DropOut & $p=0.5$& -- \\
\hline
MBDisc \cite{MBDISC}& -- & --& -- \\
\hline
Output & Dense & 1& sigmoid \\
\hline
\end{tabular}
\end{table}

\begin{table}
\caption{the classifier structure for the CFAR10 experiment.}
\label{tab:cfar3}       
\begin{tabular}{|l|l|l|l|}
\hline
Layer & Type & kernel & Activation  \\
\hline
Input & Input & $32\times32\times3$ & -- \\
\hline
Hidden 1 & Conv & $5\times5 $ (128ch)& ReLU \\
\hline
BatchNorm 1 & -- & -- &--\\
\hline
MaxPool 1 & MaxPool & (2,2)& -- \\
\hline
Hidden 2 & Conv & $5\times5 $ (256ch)& ReLU \\
\hline
BatchNorm 2 & -- & -- &--\\
\hline
MaxPool 2 & MaxPool & (2,2)& -- \\
\hline
Hidden 3 & Conv & $5\times5 $ (512ch)& ReLU \\
\hline
BatchNorm 3 & -- & -- &--\\
\hline
MaxPool 3 & MaxPool & (2,2)& -- \\
\hline
Hidden 4 & Dense & 512 & ReLU\\
\hline
Output & Dense & 10& softmax \\
\hline
\end{tabular}
\end{table}
The loss function used to train the VAC+GAN is given by
\begin{equation}
\begin{split}
&L_g = \vartheta \cdot BCE(G(z|c),1) + \zeta \cdot CCE\\
&L_g = BCE(x,1)+BCE(G(z|c),0)
\end{split}
\end{equation}
where, $L_g$, and $L_d$ are the generator and discriminator losses respectively, $G$ is the generator function, $BCE$ is the binary cross-entropy loss for discriminator and $CCE$ is the categorical cross-entropy loss for the classifier. In this experiment, $\vartheta$ and $\zeta$ are equal to 0.5 and 0.5 respectively.\\
The optimizer used for training the generator and discriminator is ADAM with learning rate, $\beta_1$ and $\beta_2$ equal to 0.0002, 0.5 and 0.999 respectively. And the classifier is optimized using nestrov momentum gradient descent with learning rate and momentum equal to 0.01 and 0.9 respectively. The results for ACGAN and proposed method are shown in figures \ref{fig:cfar1}\footnote{https://github.com/King-Of-Knights/Keras-ACGAN-CIFAR10/blob/master/plot\_epoch\_220\_generated.png} and \ref{fig:cfar2} respectively.
\begin{figure}
  \includegraphics[width=\columnwidth]{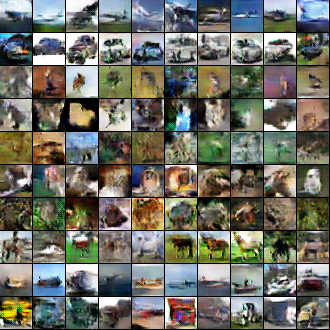}
\caption{Generated samples using ACGAN.}
\label{fig:cfar1}       
\end{figure}
\begin{figure}
  \includegraphics[width=\columnwidth]{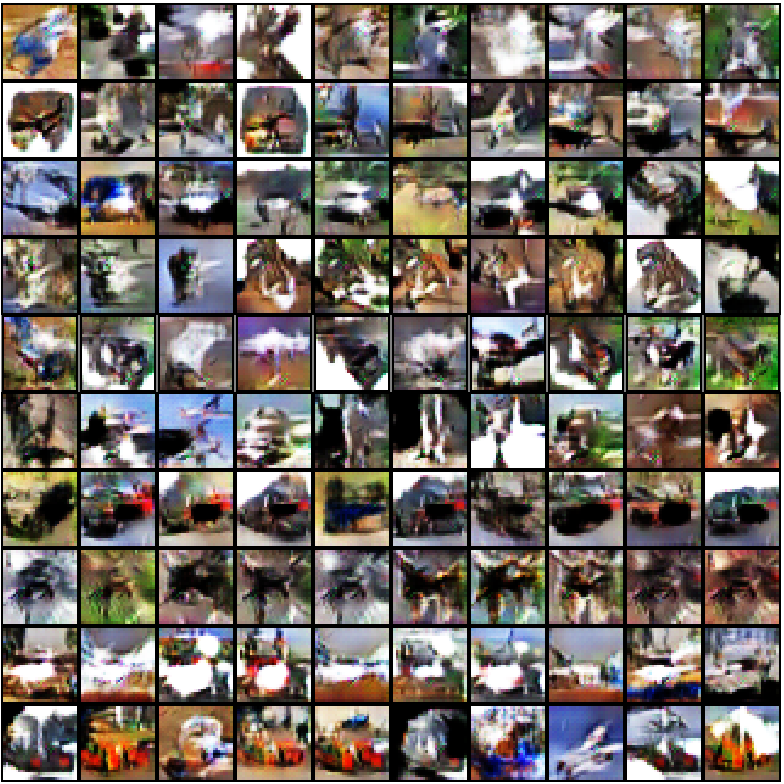}
\caption{Generated samples using VAC+GAN.}
\label{fig:cfar2}       
\end{figure}
The CFAR10 database is an extremely unconstrained and there are just 10000 samples in each class. Therefore the output of both implementations are vague and in order to compare these methods the classification errors are compared. The confusion matrix for ACGAN and VAC+GAN are shown in figures \ref{fig:cfar3}\footnote{https://github.com/King-Of-Knights/Keras-ACGAN-CIFAR10/blob/master/Confusion\_Matrix.png} and \ref{fig:cfar4} respectively.
\begin{figure}
  \includegraphics[width=\columnwidth]{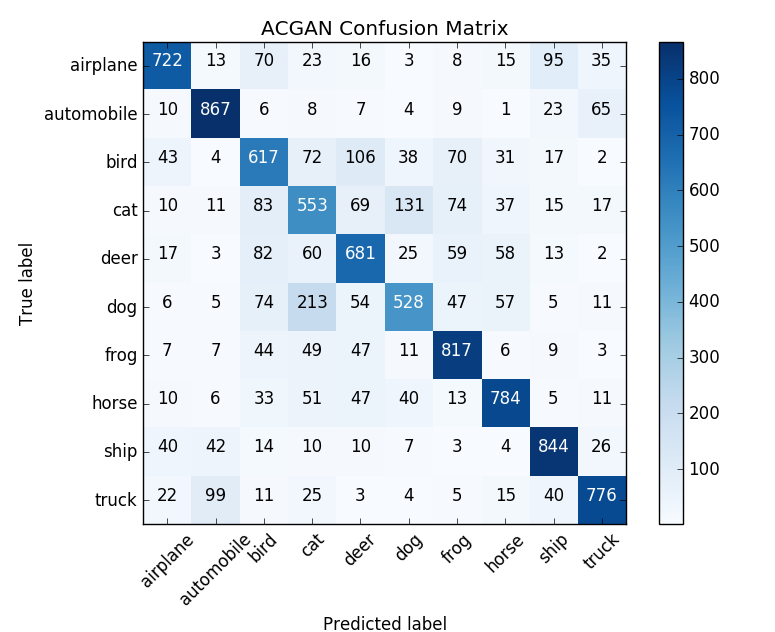}
\caption{Confusion matrix for ACGAN method on CFAR10.}
\label{fig:cfar3}       
\end{figure}
\begin{figure}
  \includegraphics[width=\columnwidth]{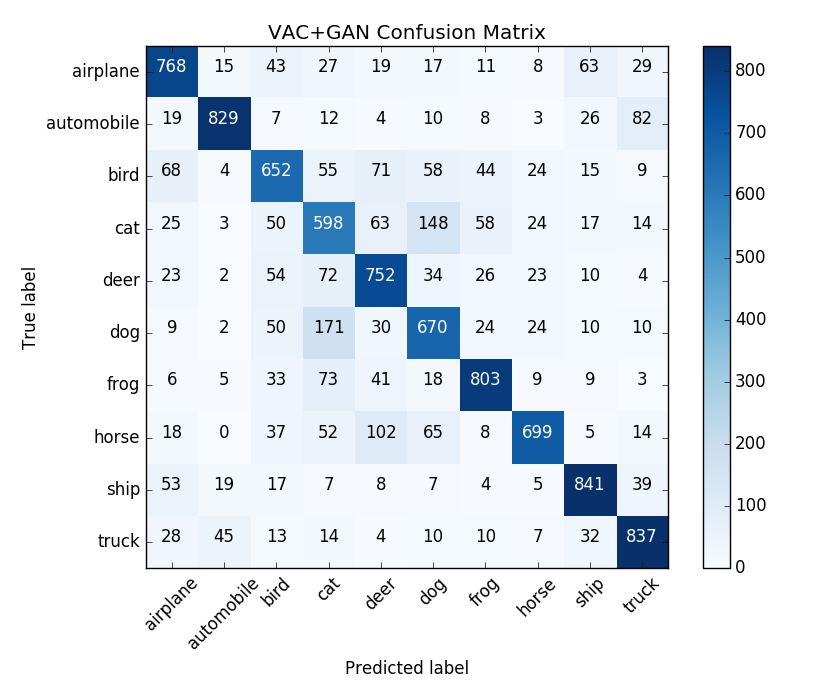}
\caption{Confusion matrix for VAC+GAN method on CFAR10.}
\label{fig:cfar4}       
\end{figure}
Confusion matrices show the better classification performed by the VAC+GAN compared to the ACGAN. Classification accuracies for ACGAN and VAC+GAN on CFAR10 are $71.89\%$ and $74.49\%$ respectively after 200 epochs. The proposed method gives higher accuracy. The main advantage of the proposed method is the versatility in choosing the proper classifier network while in the ACGAN method the classification task is restrained to discriminator because the discriminator is performing as classifier as well. The VAC+GAN method is versatile in choosing the GAN scheme as well. It can be applied to any GAN implementation just by placing a classifier in parallel with discriminator.

\section{Discussion and Conclusion}
\label{sec:4}
In this work, a new approach introduced to train conditional deep generators. It also has been proven that VAC+GAN is applicable to any GAN framework regardless of the model structure and/or loss function (see Sec \ref{sec:2}) for multi class problems. The idea is to place a classifier in parallel to the discriminator network and back-propagate the classification loss through the generator network in the training stage.\\
It has also been shown that the presented framework increases the Jensen Shannon Divergence (JSD) between classes generated by the deep generator. i.e., the generator can produce more distinct samples for different classes which is desirable.\\
The results has been compared to the implementation of CGAN, CDCGAN and ACGAN on MNIST dataset and also the comparisons are given on CFAR10 dataset with respect to ACGAN method. The ACGAN gives comparable results, but the main advantage of the proposed method is its versatility in choosing the GAN scheme and also the classifier architecture.\\
The future work includes applying the method to datasets with larger number of classes and also extend the implementation for bigger size images. The other idea is to apply this method to regression problems

\begin{acknowledgements}
This research is funded under the SFI Strategic Partnership Program by Science Foundation Ireland (SFI) and FotoNation Ltd. Project ID: 13/SPP/I2868 on Next Generation Imaging for Smartphone and Embedded Platforms.
\end{acknowledgements}

\bibliographystyle{spmpsci}      
\bibliography{lib}   
%
%

\end{document}